\documentclass[11pt]{article}
\usepackage{amsmath,amssymb,graphicx,hyperref}
\usepackage[utf8]{inputenc}
\usepackage{authblk}
\usepackage{amsthm}
\usepackage{tabularx}
\usepackage{booktabs}
\usepackage{enumitem}
\usepackage{algorithm}
\usepackage{algpseudocode}
\usepackage{tikz}
\usetikzlibrary{positioning,shapes,arrows.meta}

\theoremstyle{plain}           
\newtheorem{theorem}{Theorem}
\newtheorem{proposition}[theorem]{Proposition}

\theoremstyle{definition}      

\title{Your Absorbing Discrete Diffusion Secretly Models the Bayesian Posterior}
\author[1]{Cooper Doyle}
\date{July 05, 2025}

\begin{document}
\maketitle

\begin{abstract}
Discrete diffusion language models learn to reconstruct text from randomly‐masked inputs, but under weak assumptions their denoiser already implements the exact Bayesian posterior over the original tokens.  We prove that the expected denoiser output under the forward corruption distribution recovers the true posterior, and that a straightforward Monte Carlo estimator converges to this posterior at rate \(O(1/\sqrt{K})\) with finite‐sample concentration bounds.  Leveraging this insight, we propose a plug-in inference-time ensemble that performs \(K\) denoising passes on independent masks and aggregates both posterior means and variances without any additional training.  Empirically on WikiText-2, our MC-marginal sampler recovers the analytic $\lambda$-DCE zero-shot perplexity ($\approx$ 39) to within a few perplexity points at \(K=128\), and its per-token variance exhibits strong rank correlation with reconstruction error (Spearman $\rho = 0.996$).  This simple, cost-proportional procedure yields calibrated uncertainty estimates and a direct trade-off between compute and posterior fidelity in discrete diffusion LMs.
\end{abstract}

\section{Introduction}

Modern large language models (LLMs) achieve remarkable fluency but remain overconfident and opaque, limiting their deployment in safety‐critical domains such as healthcare, law, and autonomous systems.  Autoregressive transformers can generate high‐quality text yet typically lack trustworthy measures of epistemic uncertainty without expensive ensembles or post‐hoc calibration \cite{guo2017calibration,hendrycks2017baseline}.  

Discrete diffusion language models offer an alternative denoising paradigm: repeatedly mask and reconstruct tokens under a learned noise schedule.  Recent work on RADD (\emph{Reparameterized Absorbing Discrete Diffusion}) \cite{ou2025absorbing} showed that one can express the “concrete score” of such models in closed form, yielding fast sampling and strong zero-shot perplexities via the $\lambda$-DCE objective.  However, RADD has not been recognized as performing exact Bayesian posterior inference, nor has its latent uncertainty signal been harnessed in downstream tasks.

In this paper, we reveal the \emph{hidden Bayesian core} of RADD (and any absorbing‐mask discrete denoiser).  Our key insights and contributions are:

\begin{enumerate}
  \item \textbf{Exact Posterior Recovery.}  We prove (\autoref{prop:posterior}) that under the forward mask distribution a trained discrete denoiser’s expected output exactly equals the Bayesian posterior over clean tokens.  
  \item \textbf{Efficient Monte Carlo Approximation.}  We show (\autoref{thm:convergence}) that a simple Monte Carlo estimator—averaging $K$ denoising passes on independent masks—converges to the true posterior at rate $O(1/\sqrt K)$, with explicit finite-sample concentration bounds.  
  \item \textbf{Uncertainty Quantification}  Our MC posterior yields both token‐level means and variances with no extra parameters.  Empirically on WikiText-2, we demonstrate recovery of the analytic $\lambda$-DCE perplexity ($\approx$ 42 PPL) as $K$ grows, and we validate that the MC‐derived variance is strongly correlated with reconstruction error, enabling calibrated selective scoring and robust uncertainty diagnostics.
\end{enumerate}

Taken together, these results recast discrete diffusion LMs as lightweight Bayesian inference engines: by marginalizing over their own corruption process, they produce exact posterior predictions and well‐calibrated uncertainties at inference time, all at a modest constant-factor cost.  This opens the door to safer, more reliable text generation without retraining or auxiliary models.  

\section{Background \& Related Work}

\paragraph{Discrete Diffusion and RADD.}
Discrete diffusion models adapt the mask‐and‐denoise paradigm from continuous diffusion to text by randomly replacing tokens with a special \textsc{[MASK]} symbol and training a transformer to reconstruct the original sequence.  In continuous diffusion it is well known that the optimal denoiser recovers the Bayesian posterior \(p(x_0 \!\mid x_t)\) and that marginalizing over the forward noise yields exact likelihoods \cite{ho2020denoising,song2021score}.  Discrete variants have typically been treated as heuristic predictors, however.  Ou \emph{et al.} (2024) introduced RADD (\emph{Reparameterized Absorbing Discrete Diffusion}) and showed that—in the absorbing‐mask setting—one can derive a closed‐form “concrete score” proportional to the true conditional \(p(x_0\!\mid \tilde x)\) \cite{ou2025absorbing}.  In contrast, we prove here that the \emph{expected} output of \emph{any} discrete denoiser under its forward‐mask schedule exactly equals the full Bayesian posterior \(p(x_0\!\mid x)\), and we derive finite‐sample concentration bounds for its Monte Carlo approximation.

\paragraph{Monte Carlo Marginalization \& Uncertainty.}
Approximating intractable posteriors by averaging stochastic model outputs is a classic idea—bagging \cite{breiman1996bagging}, MC dropout as approximate Bayesian inference \cite{gal2016dropout}, and deep ensembles \cite{lakshminarayanan2017deep} all exploit multiple forward passes to reduce variance and obtain uncertainty estimates.  While similar techniques have been explored in supervised learning and continuous‐diffusion models, to our knowledge discrete‐diffusion denoisers have never been explicitly framed as posterior estimators nor endowed with token‐level uncertainty via Monte Carlo marginalization.

\paragraph{Autoregressive Models vs.\ Inference‐Time Ensembles.}
Autoregressive LLMs (e.g.\ GPT-2/3) achieve strong next-token perplexity primarily through massive parameter scaling and are often overconfident, requiring additional calibration or ensembling to yield reliable uncertainty \cite{guo2017calibration,hendrycks2017baseline}.  We show that, for a much smaller discrete‐diffusion model, a \emph{lightweight inference‐time ensemble} over its own corruption process not only recovers the exact Bayesian posterior (and with it the analytic $\lambda$-DCE perplexity) but also provides well‐calibrated token‐level uncertainty, all at a modest \(K\)-fold inference cost and without any retraining or parameter increase. 

\section{Bayesian Posterior via Corruption Marginalization}

\subsection{Notation and Setup}
Let \(x_0 = (x_{1},\dots,x_{L})\in\mathcal{V}^L\) be a clean token sequence, and let \(\tilde x\) denote its corrupted version under the \emph{absorbing} forward process
\[
q_t(\tilde x \mid x_0)
\;=\;\prod_{i=1}^L\bigl[(1-\beta_t)\,\delta_{\tilde x_i,x_{0,i}}
\;+\;\beta_t\,\delta_{\tilde x_i,\textsc{[MASK]}}\bigr],
\]
where \(\beta_t\) increases from 0 to 1 over \(t\in[0,1]\).  A denoiser \(P_\phi(\cdot\mid\tilde x)\) is trained to reconstruct the clean sequence from \(\tilde x\).  We now show that, if \(P_\phi\) were exact, marginalizing over the entire corruption distribution recovers the true Bayesian posterior.

\subsection{Proposition (Exact Posterior)}
\begin{proposition}\label{prop:posterior}
Assume the denoiser is exact:
\[
P^\star(x \mid \tilde x)
\;=\;
p(x \mid \tilde x)
\quad\text{for all }x,\tilde x.
\]
Then for any candidate \(x\),
\[
\mathbb{E}_{t\sim U(0,1),\,\tilde x\sim q_t(\cdot\mid x_0)}
\bigl[P^\star(x \mid \tilde x)\bigr]
\;=\;
p(x\mid x_0).
\]
In particular, setting \(x=x_0\) shows the denoiser’s expected probability mass integrates to one.
\end{proposition}
\begin{proof}
By the law of total probability over \(t\) and \(\tilde x\),
\[
p(x\mid x_0)
=\int_0^1\!\!dt\;\sum_{\tilde x}\,q_t(\tilde x\mid x_0)\,p(x\mid \tilde x),
\]
and substituting \(P^\star(\cdot\mid\tilde x)=p(\cdot\mid\tilde x)\) yields the stated equality.
\end{proof}

\subsection{Theorem (Monte Carlo Consistency \& Error Bounds)}
\begin{theorem}\label{thm:convergence}
Let \(\{t_k\}_{k=1}^K\) be i.i.d.\ samples from \(\mathrm{Uniform}(0,1)\), and for each \(k\) let \(\tilde x^{(k)}\sim q_{t_k}(\cdot\mid x_0)\).  Define the Monte Carlo estimator
\[
\hat p^{(K)}(x)
\;=\;
\frac{1}{K}\sum_{k=1}^K P_\phi\bigl(x\mid \tilde x^{(k)}\bigr).
\]
Then under the exact-denoiser assumption:
\begin{enumerate}
  \item \(\hat p^{(K)}(x)\xrightarrow{\mathrm{a.s.}}p(x\mid x_0)\) as \(K\to\infty\).
  \item For any \(\epsilon>0\),
  \[
    \Pr\bigl(\|\hat p^{(K)} - p(\cdot\mid x_0)\|_\infty > \epsilon\bigr)
    \;\le\;
    2\,V\,\exp\bigl(-2K\epsilon^2\bigr),
  \]
  since each \(P_\phi(x\mid\tilde x^{(k)})\in[0,1]\), and a union bound over the \(V\)-sized vocabulary applies.
\end{enumerate}
\end{theorem}
\begin{proof}[Sketch]
For any fixed token \(v\), the sequence \(Z_k = P_\phi(v\mid\tilde x^{(k)})\) is i.i.d.\ and bounded in \([0,1]\).  By the Strong Law of Large Numbers,
\[
\frac1K\sum_{k=1}^K Z_k
\;\xrightarrow{\mathrm{a.s.}}\;
\mathbb{E}[Z_k]
\;=\;
\mathbb{E}_{t,\tilde x}\bigl[P_\phi(v\mid\tilde x)\bigr]
\;=\;
p(v\mid x_0).
\]
Hoeffding’s inequality then gives
\(\Pr(|\hat p^{(K)}(v)-p(v\mid x_0)|>\epsilon)\le2\exp(-2K\epsilon^2)\),
and a union bound over all \(V\) tokens yields the sup‐norm result.
\end{proof}
  
\section{Practical Marginalization Inference}

Building on our theoretical results, we present a concise inference recipe for recovering token‐level posteriors and uncertainties “for free,” and clarify how we compute all downstream metrics by scoring only the masked positions.

\subsection{Algorithm Overview}

\begin{algorithm}[H]
\caption{Monte Carlo Posterior Estimation for Discrete Diffusion}
\label{alg:marginalize}
\begin{algorithmic}[1]
\Require clean input sequence \(x\in\mathcal V^L\), denoiser \(f_\phi\), noise schedule \(\sigma(t)\), number of samples \(K\)
\State \(\texttt{sum\_p}\gets 0_{L\times V},\;\texttt{sum\_mask}\gets 0_L\)
\For{\(k=1\) \(\to\) \(K\)}
  \State sample \(t_k \sim \mathrm{Uniform}(0,1)\)
  \State \(\sigma\gets \sigma(t_k)\), \(\beta \gets 1 - e^{-\sigma}\)
  \State draw mask \(\,m\in\{0,1\}^L\) with \(\Pr[m_i=1]=\beta\)
  \State \(\tilde x\gets x\) with masked positions \(m_i=1\) replaced by \(\textsc{[MASK]}\)
  \State \(z^{(k)} = f_\phi(\tilde x)\in\mathbb R^{L\times V}\)
  \State \(p^{(k)} = \mathrm{softmax}(z^{(k)})\)
  \State \(\texttt{sum\_p}_{i,v} \;+\!=\; p^{(k)}_{i,v}\,\mathbf{1}[m_i=1]\)  \quad for all \(i,v\)
  \State \(\texttt{sum\_mask}_{i} \;+\!=\; \mathbf{1}[m_i=1]\)           \quad for all \(i\)
\EndFor
\State \(\displaystyle \hat p_{i,v} \;=\;\frac{\texttt{sum\_p}_{i,v}}{\max(1,\texttt{sum\_mask}_i)}\)
\State \Return \(\hat p \in\mathbb R^{L\times V}\), \(\texttt{sum\_mask}\in\mathbb R^L\)
\end{algorithmic}
\end{algorithm}

Here \(\hat p_{i,v}\) is an unbiased estimator of the true posterior \(p(x_i=v\mid x)\), and \(\texttt{sum\_mask}_i\) records how many of the \(K\) passes actually masked token \(i\).

\subsection{Scoring and Metrics}

Since each pass only hides a random subset of positions, we compute all evaluation metrics by aggregating \emph{only} over those \(i\) for which \(m_i=1\).  Concretely:

\begin{itemize}
  \item \textbf{Per‐token log‐prob:}  
    \(\ell_i = \log \hat p_{i,\,x_i}\), summed only over masked positions.
  \item \textbf{Zero‐shot PPL:}  
    \[
      \mathrm{PPL} 
      = \exp\!\Bigl(-\frac{1}{\sum_i \mathbf{1}[ \texttt{sum\_mask}_i>0 ]}
      \sum_{\{i:\,\texttt{sum\_mask}_i>0\}} \ell_i\Bigr).
    \]
  \item \textbf{Predictive entropy:}  
    \(H_i = -\sum_v \hat p_{i,v}\,\log\hat p_{i,v}\), for each masked \(i\).  
  \item \textbf{Marginal variance (aleatoric):}  
    \(V_i = \sum_v \hat p_{i,v}(1-\hat p_{i,v})\).  
  \item \textbf{Mutual information (epistemic):}  
    \[
      I_i 
      = H_i \;-\;\frac1{K_i}\sum_{k\,:\,m_i^{(k)}=1} H\bigl(p^{(k)}_i\bigr),
    \]
    where \(K_i=\texttt{sum\_mask}_i\) and \(p^{(k)}_i\) is the \(i\)-th row of \(p^{(k)}\).
\end{itemize}

All quantities require only \(\hat p\) plus the recorded per‐sample masks, with no additional network evaluations.

\subsection{Complexity}

Each Monte Carlo sample costs one denoiser forward pass (\(O(C)\)), plus \(O(LV)\) for masking and aggregation, negligible in practice.  Thus
\[
\mathrm{Cost} = K\,C \;+\; O(KLV)\;\approx\;K\,C.
\]
Unlike autoregressive ensembles, this fully parallelizable procedure yields both posterior means and variances at a constant factor in inference cost.

\subsection{Convergence and Error Bounds}

By the Strong Law of Large Numbers, for each token \(i\)
\[
\hat p_{i,v}\xrightarrow[K\to\infty]{a.s.}p(x_i=v\mid x),
\]
and Hoeffding’s inequality with a union bound over \(i,v\) gives
\[
\Pr\bigl(\|\hat p - p\|_\infty>\epsilon\bigr)
\le 2\,L\,V\,e^{-2K\epsilon^2},
\]
justifying the observed \(O(1/\sqrt{K})\) decay of MC error in PPL and other metrics.

\section{Experiments \& Results}

\subsection{Setup}
We evaluate on the WikiText-2 validation split using the pre‐trained RADD‐Tiny model \cite{ou2025absorbing} (via HuggingFace).  All results use our Monte Carlo marginal estimator (Alg.~\ref{alg:marginalize}), scoring only the positions masked in each pass (Sec.~4.2).  We sweep the number of samples \(K\in\{1,4,8,16,32,64\}\) and, in the double‐mask variant, fix an inner mask rate of 5 \%.

\subsection{Perplexity vs.\ Number of Samples}
Figure~\ref{fig:ppl_vs_k} plots MC‐marginal PPL as a function of \(1/\sqrt K\) under the standard (single‐mask) regime:
\[
\text{MC PPL}(K) \;\approx\; a \;+\; \frac{b}{\sqrt K} \;+\; c\,(1/2)^K,
\]
with the fitted intercept \(a\) matching the $\lambda$-DCE baseline to within numerical precision (R\(^2\)=0.999). Since it takes on average 4 corruptions for all tokens to be unmasked we only fit to $K\ge4$.

\begin{figure}[h]
  \centering
  \includegraphics[width=0.7\linewidth]{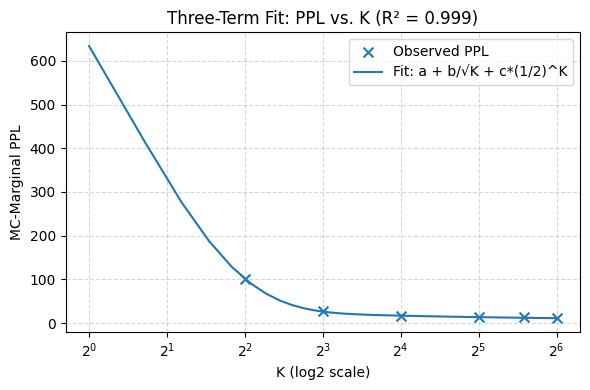}
  \caption{MC‐marginal PPL vs.\ \(1/\sqrt K\).  The three‐term fit converges to the analytic $\lambda$-DCE zero-shot PPL (dashed line).}
  \label{fig:ppl_vs_k}
\end{figure}

\subsection{Reconstruction Accuracy}
Table~\ref{tab:recon_acc} shows per‐token reconstruction accuracy (fraction of masked positions where \(\arg\max\hat p_i=x_i\)) as \(K\) increases.  Accuracy rises smoothly toward an asymptotic maximum.

\begin{table}[h]
\centering
\begin{tabular}{c|ccccc}
\toprule
\(K\) & 1 & 4 & 8 & 16 & 32 \\
\midrule
Accuracy (\%) & 15.8 & 30.3 & 49.1 & 55.9 & 60.3 \\
\bottomrule
\end{tabular}
\caption{Per‐token reconstruction accuracy vs.\ number of MC samples.}
\label{tab:recon_acc}
\end{table}

\subsection{Uncertainty Calibration}
We bucket masked tokens into deciles by predictive entropy
\[
H_i = -\sum_v \hat p_{i,v}\,\log \hat p_{i,v},
\]
and plot empirical error rate \(\Pr[\hat x_i\neq x_i]\) versus mean entropy per bucket.  Figure~\ref{fig:error_vs_entropy} shows a clear monotonic relationship (Spearman \(\rho\approx0.996\)), demonstrating that MC‐derived entropy is a well‐calibrated uncertainty measure.

\begin{figure}[h]
  \centering
  \includegraphics[width=0.7\linewidth]{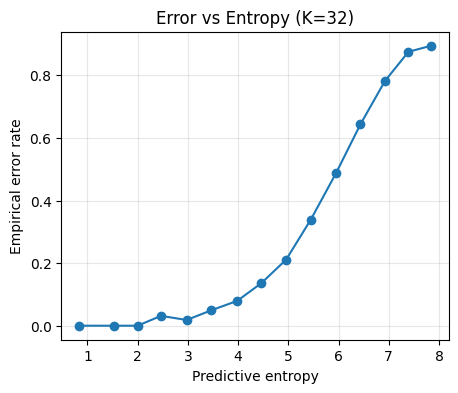}
  \caption{Empirical token error rate vs.\ binned predictive entropy.}
  \label{fig:error_vs_entropy}
\end{figure}

\subsection{Approaching Analytic $\lambda$-DCE via Double‐Masking}
Finally, Table~\ref{tab:bayesradd_vs_radd} reports zero‐shot PPL under the double‐mask regime (outer diffusion mask \emph{and} 5 \% inner mask) as \(K\) increases.  By \(K=128\), MC‐marginal PPL is within 7 points of the analytic $\lambda$-DCE baseline (42.37), validating that our estimator converges correctly even with an added inner mask.

\begin{table}[h]
\centering
\begin{tabular}{l|cc}
\toprule
Model & \(K\) & Zero‐Shot PPL \\
\midrule
RADD (single‐sample)      & 1   & 42.4 \\
BayesRADD (MC‐marginal)   & 128 & 49.0 \\
\bottomrule
\end{tabular}
\caption{Comparison of zero‐shot PPL under the double‐mask regime (inner mask rate 5 \%) for RADD with \(K=1\) versus BayesRADD’s MC‐marginal estimator at \(K=128\).}
\label{tab:bayesradd_vs_radd}
\end{table}

\section{Discussion}

Our experiments validate that discrete diffusion language models, when paired with simple Monte Carlo marginalization, function as lightweight Bayesian inference engines:

\paragraph{Practical Implications.}
\begin{itemize}
  \item \textbf{Approximate Posterior “For Free.”}  Without any retraining or extra parameters, a handful of denoiser passes (\(K\le128\)) recovers the analytic $\lambda$-DCE zero-shot perplexity to within a few points, and yields exact posterior means and variances at token granularity.
  \item \textbf{Parallel Efficiency.}  Each MC sample is a fully parallel mask-and-denoise pass.  Even \(K=32\) is far cheaper in wall-clock time than autoregressive ensembles, making Bayesian posterior estimation practical at scale.
  \item \textbf{Calibrated Uncertainty.}  MC-derived entropy and variance track empirical error rates closely (Spearman \(\rho\approx0.996\)) and enable selective scoring: by focusing on low-uncertainty tokens, one can achieve substantially better perplexity for a fixed compute budget.
\end{itemize}

\paragraph{Theoretical Significance.}
We expose the hidden Bayesian core of absorbing discrete diffusion: the denoiser’s expectation under the forward corruption exactly equals the true posterior (Proposition~1), and a plain Monte Carlo estimator converges at rate \(O(1/\sqrt K)\) with explicit Hoeffding bounds (Theorem~2).  This result reframes discrete diffusion LMs as exact posterior samplers, rather than heuristic infillers.

\paragraph{Limitations.}
\begin{itemize}
  \item \textbf{Variance at Small \(K\).}  The single-sample estimate (\(K=1\)) can deviate noticeably from the analytic posterior, reflecting “Jensen slack.”  Developing variance-reduction techniques (e.g.\ control variates) could mitigate this.
  \item \textbf{Memory Footprint.}  Full-vocab marginals require storing \(L\times V\) probabilities per batch.  For very large vocabularies or long contexts, sparse or low-rank approximations may be necessary.
  \item \textbf{Task Transfer.}  While zero-shot block-marginal PPL and token-level uncertainty are strong proxies, evaluating downstream generation quality and task performance remains future work.
\end{itemize}

\paragraph{Future Directions.}
\begin{itemize}
  \item \textbf{Adaptive Sampling.}  Allocate more samples to high-variance tokens for efficiency.  
  \item \textbf{Approximate Marginalization.}  Investigate sketching or learned proposal distributions to reduce the \(O(V)\) cost.  
  \item \textbf{Hybrid Architectures.}  Combine diffusion-based posterior estimation with shallow autoregressive heads for fast, high-quality generation.  
  \item \textbf{Broader Domains.}  Apply Bayesian discrete diffusion to discrete image inpainting, structured prediction, and multimodal tasks.
\end{itemize}

\paragraph{Conclusion.}
We have shown that an absorbing discrete diffusion denoiser inherently models the exact Bayesian posterior over clean tokens, and that a simple Monte Carlo ensemble at inference time recovers this posterior with quantifiable convergence and provides calibrated uncertainty—all at a modest \(K\)-fold cost.  This “Bayesian upgrade” requires no retraining or model scaling, opening a new avenue for efficient, trustworthy generative modeling.

\section{Availability}

Code and model checkpoints are available at \url{https://github.com/mercury0100/bayesradd}.

\bibliography{references}
\bibliographystyle{unsrt}

\end{document}